\documentclass{article}

\usepackage[english]{babel}

\usepackage[letterpaper,top=2cm,bottom=2cm,left=3cm,right=3cm,marginparwidth=1.75cm]{geometry}

\usepackage{amsmath}
\usepackage{graphicx}
\usepackage{authblk}
\usepackage{esvect}
\usepackage{amsthm}
\usepackage{tikz}
\usepackage{algorithm}
\usepackage{amsfonts}
\usepackage{subcaption}
\usepackage[noend]{algpseudocode}
\usepackage{adjustbox}

\usepackage{gnuplot-lua-tikz}

\usepackage{hyperref}
\hypersetup{colorlinks=true,allcolors=blue}

\newtheorem{theorem}{Theorem}[section]
\newtheorem{lemma}[theorem]{Lemma}

\theoremstyle{definition}
\newtheorem{definition}{Definition}[section]

\algnewcommand\Or{\textbf{or} }
\algnewcommand\Select{\textbf{select} }
\algnewcommand\Replace{\textbf{replace} }
\algnewcommand\Rep{\textbf{repeat} }

\title{Jump Point Search Pathfinding in 4-connected Grids}
\author{Johannes Baum}
\affil{Google}
\affil{Email: jobaum@google.com}
\DeclareUnicodeCharacter{2212}{-}

\usetikzlibrary{arrows.meta}
\begin{document}
\maketitle

\begin{abstract}
This work introduces JPS4, a novel pathfinding algorithm for 4-connected grid maps. JPS4 builds upon the Jump Point Search (JPS8) algorithm, originally designed for 8-connected environments. To achieve efficient pathfinding on 4-connected grids, JPS4 employs a canonical ordering and a successor function that enable online graph pruning. This reduces the search space by minimizing unnecessary node expansions.

The core concept of JPS4 as well as JPS8 lies in the utilization of jump points. Strategically placed at obstacle corners, jump points prevent the search from overlooking crucial sections of the state space. They essentially reinitialize the canonical ordering, allowing exploration beyond obstacles. This mechanism ensures JPS4 finds optimal paths even in complex environments.

The paper further explores the optimality of JPS4 and compares its performance against the established A* algorithm on various grid maps. Benchmarking results demonstrate that JPS4 significantly outperforms A* in scenarios with high obstacle density. However, A* remains more efficient on open maps. Overall, JPS4 presents itself as a promising alternative to A* for pathfinding on 4-connected grids, particularly applicable in video game development.
\end{abstract}

\section{Introduction}
Pathfinding on a 4-connected grid map is a common problem in video games \cite{botea_et_al}\cite{Harabor_Botea_2010} and robotics \cite{Hu_Harabor_Gange_Stuckey_Sturtevant_2021}. The Jump Point Search (JPS8) algorithm \cite{harabor11} is an adapted version of A* \cite{A_star}, adding a specific canonical ordering and successor function \cite{canonical_orderings}. It improves the performance of traditional A* on 8-connected uniform cost grid maps without any precomputations.
This work adapts the original JPS8 algorithm to work on 4-connected uniform cost grid maps. JPS8 prunes the graph online following specified pruning rules that define the canonical ordering. We adapt the pruning rules to produce a horizontal-first canonical ordering that can be applied to 4-connected grid maps.

\section{Related Work}

Optimizations of best-first search pathfinding algorithms like A* usually improve the performance by adapting the heuristic, introducing hierarchical abstractions or eliminating path symmetries \cite{botea_et_al}.

Hierarchical pathfinding algorithms like HPA* \cite{hpa} improve the performance while sacrificing optimality. 

One approach to break path symmetries on 4-connected grid maps identifies rectangular rooms without obstacles and prunes all interior nodes \cite{Harabor_Botea_2010}. This is done offline in a precomputation step.

The original Jump Point Search (JPS8) algorithm \cite{harabor11} improves traditional A* on 8-connected uniform cost grid maps without any precomputations.

Hu et al. provide a canonical ordering (Vertical-then-Horizontal-then-Wait) for JPS8 that works on 4-connected grid maps with temporal obstacles.

\section{Notation and Terminology}

We refer to the original algorithm that was proposed in \cite{harabor11} as JPS8 (Jump Point Search). While JPS8 allows 8 movement directions (left, right, up, down, up-left, up-right, down-left, down-right), our proposed algorithm adapts JPS8 to allow only 4 (cardinal) movement directions (left, right, up, down).
Therefore we are calling our proposed algorithm JPS4.

As in the original JPS8 the algorithm works on undirected uniform-cost grid maps.
\\
\\
\textbf{Directions}:
$D = \{up, down, left, right\}$ denotes the four allowed cardinal movement directions. $H = \{left, right\} \subset D$ is the set of \textbf{horizontal} directions.
$V = \{up,down\} \subset D$ is the set of \textbf{vertical} directions.
\\
\\
\textbf{Nodes}: Let $N$ denote the set of all nodes in our grid. Each node is either traversable or not and a move from a traversable node to one of its neighbors always has a cost of 1. Moves from or to non-traversable nodes are not allowed. 
$neighbors: N \rightarrow \mathcal{P}(N)$
 provides all traversable neighbors in directions $D$. 
\\
\\

As in \cite{harabor11}, we write $y = x+kd$ when node $y \in N$ can be reached
by taking $k \in \mathbb{N}$ unit moves from node $x \in N$ in direction $d \in D$.
\\
\\
\textbf{Paths}:
We are using the path notation from \cite{harabor11}:
A path
$\pi = \langle x_0, x_1, \dots , x_k \rangle$
is a cycle-free ordered walk starting at node $x_0$ and ending at $x_k$. The notation $\pi \setminus x$ denotes a path that does not contain node $x \in N$.
The set of all such paths is called $\Pi$.
$len: \Pi \rightarrow \mathbb{N}$
 is a function that defines the length of a path. $direction: N \times N \rightarrow D$
 provides the direction moving from one node to its neighbor.

\section{Algorithm}
A key concept of the original version of Jump Point Search \cite{harabor11} as well as this adapted version is a so called $jump$ $point$. Jump points help to leverage path symmetries in order to reduce the total amount of nodes that need to be expanded by the search algorithm.

While classical A* expands all neighbors for a node, we look ahead in the current movement direction for vertical movements only. This means that when moving up, we only check the neighbor in direction $up$ and when moving down only check the neighbor in direction $down$. This works fine unless we see an obstacle. In this case we also have to expand in other directions. We formalize these ideas in later sections.

\subsection{Canonical Ordering}

On open maps without obstacles there are many optimal paths from one position to another. Consider figure \ref{fig:path-symmetry}. From point A to B there are plenty of optimal paths. In that case each path with 4 right movements and 4 up movements leads to target B. The order of the right and up movements does not matter. This redundancy among the optimal paths is called path symmetries. A canonical ordering is a total ordering over paths which can eliminate many of the redundant paths \cite{canonical_orderings}. We chose the horizontal first canonical ordering, which prefers horizontal movements before vertical ones. With this canonical ordering the search space is no longer a graph but a tree \cite{canonical_orderings} as shown in figure \ref{fig:can-order-tree}.
In figure \ref{fig:can-order-path} this canonical ordering provides the shown path, that first contains all horizontal movements and then all vertical ones.

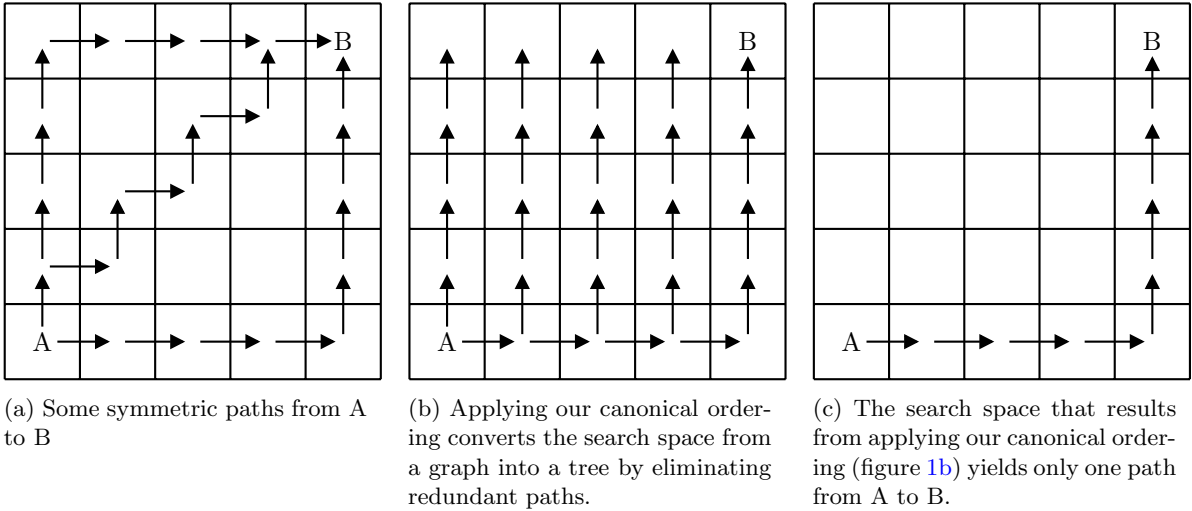
\begin{figure}
    \begin{subfigure}[t]{0.31\textwidth}
        \vskip 0pt
        \begin{tikzpicture}[scale=1, thick]
          \begin{scope}
            \draw (0, 0) grid (5, 5);
            \node at (0.5,0.5) {A};
            \node at (4.5,4.5) {B};
            \draw[->, arrows={-Triangle[length=0.2cm]}] (0.7,0.5) -- (1.4,0.5);
            \draw[->, arrows={-Triangle[length=0.2cm]}] (1.6,0.5) -- (2.4,0.5);
            \draw[->, arrows={-Triangle[length=0.2cm]}] (2.6,0.5) -- (3.4,0.5);
            \draw[->, arrows={-Triangle[length=0.2cm]}] (3.6,0.5) -- (4.4,0.5);
            \draw[->, arrows={-Triangle[length=0.2cm]}] (0.5,0.7) -- (0.5,1.4);
            \draw[->, arrows={-Triangle[length=0.2cm]}] (0.5,1.6) -- (0.5,2.4);
            \draw[->, arrows={-Triangle[length=0.2cm]}] (0.5,2.6) -- (0.5,3.4);
            \draw[->, arrows={-Triangle[length=0.2cm]}] (0.5,3.6) -- (0.5,4.4);
            \draw[->, arrows={-Triangle[length=0.2cm]}] (0.6,4.5) -- (1.4,4.5);
            \draw[->, arrows={-Triangle[length=0.2cm]}] (1.6,4.5) -- (2.4,4.5);
            \draw[->, arrows={-Triangle[length=0.2cm]}] (2.6,4.5) -- (3.4,4.5);
            \draw[->, arrows={-Triangle[length=0.2cm]}] (3.6,4.5) -- (4.35,4.5);
            \draw[->, arrows={-Triangle[length=0.2cm]}] (4.5,0.6) -- (4.5,1.4);
            \draw[->, arrows={-Triangle[length=0.2cm]}] (4.5,1.6) -- (4.5,2.4);
            \draw[->, arrows={-Triangle[length=0.2cm]}] (4.5,2.6) -- (4.5,3.4);
            \draw[->, arrows={-Triangle[length=0.2cm]}] (4.5,3.6) -- (4.5,4.3);
            \draw[->, arrows={-Triangle[length=0.2cm]}] (0.6,1.5) -- (1.4,1.5);
            \draw[->, arrows={-Triangle[length=0.2cm]}] (1.6,2.5) -- (2.4,2.5);
            \draw[->, arrows={-Triangle[length=0.2cm]}] (2.6,3.5) -- (3.4,3.5);
            \draw[->, arrows={-Triangle[length=0.2cm]}] (1.5,1.6) -- (1.5,2.4);
            \draw[->, arrows={-Triangle[length=0.2cm]}] (2.5,2.6) -- (2.5,3.4);
            \draw[->, arrows={-Triangle[length=0.2cm]}] (3.5,3.6) -- (3.5,4.4);
          \end{scope}
        \end{tikzpicture}
        \caption{Some symmetric paths from A to B}
        \label{fig:path-symmetry}
    \end{subfigure}
    \hspace*{\fill}
    \begin{subfigure}[t]{0.31\textwidth}
        \vskip 0pt
        \begin{tikzpicture}[scale=1, thick]
          \begin{scope}
            \draw (0, 0) grid (5, 5);
            \node at (0.5,0.5) {A};
            \node at (4.5,4.5) {B};
            \draw[->, arrows={-Triangle[length=0.2cm]}] (0.7,0.5) -- (1.4,0.5);
            \draw[->, arrows={-Triangle[length=0.2cm]}] (1.6,0.5) -- (2.4,0.5);
            \draw[->, arrows={-Triangle[length=0.2cm]}] (2.6,0.5) -- (3.4,0.5);
            \draw[->, arrows={-Triangle[length=0.2cm]}] (3.6,0.5) -- (4.4,0.5);
            
            \draw[->, arrows={-Triangle[length=0.2cm]}] (0.5,0.7) -- (0.5,1.4);
            \draw[->, arrows={-Triangle[length=0.2cm]}] (0.5,1.6) -- (0.5,2.4);
            \draw[->, arrows={-Triangle[length=0.2cm]}] (0.5,2.6) -- (0.5,3.4);
            \draw[->, arrows={-Triangle[length=0.2cm]}] (0.5,3.6) -- (0.5,4.4);
            \draw[->, arrows={-Triangle[length=0.2cm]}] (4.5,0.6) -- (4.5,1.4);
            \draw[->, arrows={-Triangle[length=0.2cm]}] (4.5,1.6) -- (4.5,2.4);
            \draw[->, arrows={-Triangle[length=0.2cm]}] (4.5,2.6) -- (4.5,3.4);
            \draw[->, arrows={-Triangle[length=0.2cm]}] (4.5,3.6) -- (4.5,4.3);
            
            \draw[->, arrows={-Triangle[length=0.2cm]}] (1.5,0.6) -- (1.5,1.4);
            \draw[->, arrows={-Triangle[length=0.2cm]}] (1.5,1.6) -- (1.5,2.4);
            \draw[->, arrows={-Triangle[length=0.2cm]}] (1.5,2.6) -- (1.5,3.4);
            \draw[->, arrows={-Triangle[length=0.2cm]}] (1.5,3.6) -- (1.5,4.4);
            
            \draw[->, arrows={-Triangle[length=0.2cm]}] (2.5,0.6) -- (2.5,1.4);
            \draw[->, arrows={-Triangle[length=0.2cm]}] (2.5,1.6) -- (2.5,2.4);
            \draw[->, arrows={-Triangle[length=0.2cm]}] (2.5,2.6) -- (2.5,3.4);
            \draw[->, arrows={-Triangle[length=0.2cm]}] (2.5,3.6) -- (2.5,4.4);
            
            \draw[->, arrows={-Triangle[length=0.2cm]}] (3.5,0.6) -- (3.5,1.4);
            \draw[->, arrows={-Triangle[length=0.2cm]}] (3.5,1.6) -- (3.5,2.4);
            \draw[->, arrows={-Triangle[length=0.2cm]}] (3.5,2.6) -- (3.5,3.4);
            \draw[->, arrows={-Triangle[length=0.2cm]}] (3.5,3.6) -- (3.5,4.4);
          \end{scope}
        \end{tikzpicture}
        \caption{Applying our canonical ordering converts the search space from a graph into a tree by eliminating redundant paths.}
        \label{fig:can-order-tree}
    \end{subfigure}
    \hspace*{\fill}
    \begin{subfigure}[t]{0.31\textwidth}
        \vskip 0pt
        \begin{tikzpicture}[scale=1, thick]
          \begin{scope}
            \draw (0, 0) grid (5, 5);
            \node at (0.5,0.5) {A};
            \node at (4.5,4.5) {B};
            \draw[->, arrows={-Triangle[length=0.2cm]}] (0.7,0.5) -- (1.4,0.5);
            \draw[->, arrows={-Triangle[length=0.2cm]}] (1.6,0.5) -- (2.4,0.5);
            \draw[->, arrows={-Triangle[length=0.2cm]}] (2.6,0.5) -- (3.4,0.5);
            \draw[->, arrows={-Triangle[length=0.2cm]}] (3.6,0.5) -- (4.4,0.5);
            
            \draw[->, arrows={-Triangle[length=0.2cm]}] (4.5,0.6) -- (4.5,1.4);
            \draw[->, arrows={-Triangle[length=0.2cm]}] (4.5,1.6) -- (4.5,2.4);
            \draw[->, arrows={-Triangle[length=0.2cm]}] (4.5,2.6) -- (4.5,3.4);
            \draw[->, arrows={-Triangle[length=0.2cm]}] (4.5,3.6) -- (4.5,4.3);
          \end{scope}
        \end{tikzpicture}
        \caption{The search space that results from applying our canonical ordering (figure \ref{fig:can-order-tree}) yields only one path from A to B.}
        \label{fig:can-order-path}
    \end{subfigure}
    \caption{Elimination of symmetric paths in obstacle-free maps via canonical ordering.}
\end{figure}

\subsection{Neighbor Pruning Rules}
In the previous section, we said that we are not going to expand every neighbor for vertical movements.
To achieve this, we prune the neighbors in a way that does not influence the optimality of the resulting path. By pruning, we apply our canonical ordering to the search space.

In order to decide which neighbors of a node $x$ to prune, we need to know the parent node (the adjacent node we are coming from). So, pruning depends on the movement direction. We denote the parent node of $x$ by $p$: $N \rightharpoonup N$ where $\forall x \in N$: $p(x)$ is not defined or $p(x) \in neighbors(x) $.
If $x$ is the start node and there exists no $p(x)$, nothing is pruned.

\begin{description}
   \item[Horizontal:] We don't prune any nodes.
   \item[Vertical:] Ignoring obstacles, we prune any node $n \in neighbors(x)$ that satisfies $len( \langle p(x), \dots , n \rangle \setminus x ) \leq len( \langle p(x), x, n \rangle )$. That means we prune any neighbor of $x$ that can be reached on a path that does not include $x$ and is not longer than any path that does include $x$. As a result, the left and right neighbors of $x$ are pruned as demonstrated in figure \ref{fig:nat-neighs-vertical}.
\end{description}

We define two kinds of neighbors: natural neighbors and forced neighbors. Natural neighbors are those that remain after pruning. Forced neighbors are those that would usually be pruned, but they are not because there is an obstacle (non-traversable node) around.

\begin{definition}[Natural Neighbors]\label{def:natural}
A node $n \in neighbors(x)$ is called a natural neighbor if it is not pruned according to the above pruning rules.
\end{definition}

\begin{definition}[Forced Neighbors]\label{def:forced}
A node $n \in neighbors(x)$ is forced if:
\begin{enumerate}
\item $n$ is not a natural neighbor of $x$
\item $len( \langle p(x), \ldots , n \rangle \setminus x )  > len( \langle $p(x)$, x, n \rangle )$
\end{enumerate}
\end{definition}

Applying the pruning rules yields the natural neighbors shown in figure \ref{fig:nat-neighs-vertical} for a node $x$ when coming from parent $p$ in a vertical movement. This means that there is only one natural neighbor for vertical movements (the one in the same direction that we are moving towards).

\begin{figure}
    \centering
    \begin{subfigure}[t]{0.40\textwidth}
        \centering
        \vskip 0pt
        \begin{tikzpicture}[scale=1, thick]
          \begin{scope}
            \fill[lightgray] (0,0) -- +(0, 1) -- +(1, 1) -- +(1,0) -- cycle;
            \fill[lightgray] (0,2) -- +(0, 1) -- +(1, 1) -- +(1,0) -- cycle;
            \fill[lightgray] (2,2) -- +(0, 1) -- +(1, 1) -- +(1,0) -- cycle;
            \fill[lightgray] (2,0) -- +(0, 1) -- +(1, 1) -- +(1,0) -- cycle;
            \fill[lightgray] (2,1) -- +(0, 1) -- +(1, 1) -- +(1,0) -- cycle;
            \fill[lightgray] (1,0) -- +(0, 1) -- +(1, 1) -- +(1,0) -- cycle;
            \fill[lightgray] (1,1) -- +(0, 1) -- +(1, 1) -- +(1,0) -- cycle;
            \fill[lightgray] (0,1) -- +(0, 1) -- +(1, 1) -- +(1,0) -- cycle;
            \draw (0, 0) grid (3, 3);
            \node at (1.5,0.5) {p};
            \node at (1.5,1.5) {x};
            \node at (1.5,2.5) {nn};
            \draw[->, arrows={-Triangle[length=0.2cm]}] (1.5,0.7) -- (1.5,1.4);
          \end{scope}
        \end{tikzpicture}
        \caption{There is only one natural neighbor for a vertical movement.}
        \label{fig:nat-neighs-vertical}
    \end{subfigure}
    \hspace*{2cm}   
    \begin{subfigure}[t]{0.40\textwidth}
        \centering
        \vskip 0pt
        \begin{tikzpicture}[scale=1, thick]
          \begin{scope}
            \fill[lightgray] (0,0) -- +(0, 1) -- +(1, 1) -- +(1,0) -- cycle;
            \fill[lightgray] (0,2) -- +(0, 1) -- +(1, 1) -- +(1,0) -- cycle;
            \fill[lightgray] (2,2) -- +(0, 1) -- +(1, 1) -- +(1,0) -- cycle;
            \fill[lightgray] (2,0) -- +(0, 1) -- +(1, 1) -- +(1,0) -- cycle;
            \fill[lightgray] (1,0) -- +(0, 1) -- +(1, 1) -- +(1,0) -- cycle;
            \fill[lightgray] (1,1) -- +(0, 1) -- +(1, 1) -- +(1,0) -- cycle;
            \draw (0, 0) grid (3, 3);
            \node at (1.5,0.5) {p};
            \node at (1.5,1.5) {x};
            \node at (1.5,2.5) {nn};
            \node at (0.5,1.5) {fn?};
            \node at (2.5,1.5) {fn?};
            \draw[->, arrows={-Triangle[length=0.2cm]}] (1.5,0.7) -- (1.5,1.4);
          \end{scope}
        \end{tikzpicture}
        \caption{All neighbors that are not natural neighbors in a vertical movement are potentially forced neighbors.}
        \label{fig:pot-forced-vertical}
    \end{subfigure}

            \vskip 10pt
    \begin{subfigure}[t]{1\textwidth}
        \centering
        \hspace*{\fill}
        \begin{subfigure}[t]{0.30\textwidth}
            \centering
            \vskip 0pt
            \begin{tikzpicture}[scale=1, thick]
              \begin{scope}
                \fill[black] (0,0) -- +(0, 1) -- +(1, 1) -- +(1,0) -- cycle;
                \fill[lightgray] (0,2) -- +(0, 1) -- +(1, 1) -- +(1,0) -- cycle;
                \fill[lightgray] (2,2) -- +(0, 1) -- +(1, 1) -- +(1,0) -- cycle;
                \fill[lightgray] (2,0) -- +(0, 1) -- +(1, 1) -- +(1,0) -- cycle;
                \fill[lightgray] (1,0) -- +(0, 1) -- +(1, 1) -- +(1,0) -- cycle;
                \fill[lightgray] (1,1) -- +(0, 1) -- +(1, 1) -- +(1,0) -- cycle;
                \fill[lightgray] (2,1) -- +(0, 1) -- +(1, 1) -- +(1,0) -- cycle;
                \draw (0, 0) grid (3, 3);
                \node at (1.5,0.5) {p};
                \node at (1.5,1.5) {x};
                \node at (1.5,2.5) {nn};
                \node at (0.5,1.5) {fn};
                \draw[->, arrows={-Triangle[length=0.2cm]}] (1.5,0.7) -- (1.5,1.4);
              \end{scope}
            \end{tikzpicture}
        \end{subfigure}
        \hspace*{\fill}
        \begin{subfigure}[t]{0.30\textwidth}
            \centering
            \vskip 0pt
            \begin{tikzpicture}[scale=1, thick]
              \begin{scope}
                \fill[black] (2,0) -- +(0, 1) -- +(1, 1) -- +(1,0) -- cycle;
                \fill[lightgray] (0,2) -- +(0, 1) -- +(1, 1) -- +(1,0) -- cycle;
                \fill[lightgray] (2,2) -- +(0, 1) -- +(1, 1) -- +(1,0) -- cycle;
                \fill[lightgray] (1,0) -- +(0, 1) -- +(1, 1) -- +(1,0) -- cycle;
                \fill[lightgray] (1,1) -- +(0, 1) -- +(1, 1) -- +(1,0) -- cycle;
                \fill[lightgray] (0,1) -- +(0, 1) -- +(1, 1) -- +(1,0) -- cycle;
                \fill[lightgray] (0,0) -- +(0, 1) -- +(1, 1) -- +(1,0) -- cycle;
                \draw (0, 0) grid (3, 3);
                \node at (1.5,0.5) {p};
                \node at (1.5,1.5) {x};
                \node at (1.5,2.5) {nn};
                \node at (2.5,1.5) {fn};
                \draw[->, arrows={-Triangle[length=0.2cm]}] (1.5,0.7) -- (1.5,1.4);
              \end{scope}
            \end{tikzpicture}
        \end{subfigure}
        \hspace*{\fill}
        \begin{subfigure}[t]{0.30\textwidth}
            \centering
            \vskip 0pt
            \begin{tikzpicture}[scale=1, thick]
              \begin{scope}
                \fill[black] (0,0) -- +(0, 1) -- +(1, 1) -- +(1,0) -- cycle;
                \fill[black] (2,0) -- +(0, 1) -- +(1, 1) -- +(1,0) -- cycle;
                \fill[lightgray] (0,2) -- +(0, 1) -- +(1, 1) -- +(1,0) -- cycle;
                \fill[lightgray] (2,2) -- +(0, 1) -- +(1, 1) -- +(1,0) -- cycle;
                \fill[lightgray] (1,0) -- +(0, 1) -- +(1, 1) -- +(1,0) -- cycle;
                \fill[lightgray] (1,1) -- +(0, 1) -- +(1, 1) -- +(1,0) -- cycle;
                \draw (0, 0) grid (3, 3);
                \node at (1.5,0.5) {p};
                \node at (1.5,1.5) {x};
                \node at (1.5,2.5) {nn};
                \node at (2.5,1.5) {fn};
                \node at (0.5,1.5) {fn};
                \draw[->, arrows={-Triangle[length=0.2cm]}] (1.5,0.7) -- (1.5,1.4);
              \end{scope}
            \end{tikzpicture}
        \end{subfigure}%
        \caption{Natural and forced neighbors for vertical movements in case of obstacles}
        \label{fig:forced-obstacle-vertical}
    \end{subfigure}%

    \caption{Natural and forced neighbors for vertical movements (black square:  obstacle, p: parent node, x: current node, nn: natural neighbor, fn?: potentially forced neighbor, fn: forced neighbor)}
\end{figure}

Horizontal movements never produce forced neighbors. For vertical movements we need to check the pruning condition for all neighbors that are not natural neighbors. Figure \ref{fig:pot-forced-vertical} shows all potentially forced neighbors for vertical movements.

If there is no obstacle, none of these potentially forced neighbors are actual forced neighbors. Figure \ref{fig:forced-obstacle-vertical} shows how obstacles produce forced neighbors.

Algorithms \ref{alg:prune}, \ref{alg:getnn} and \ref{alg:forcedn} formalize the pruning rules using the above definitions.

\begin{algorithm}
\caption{Function $prune$}\label{alg:prune}
\begin{algorithmic}[1]
\Require $p$: parent node, $x$: current node
\If{$p = null$}
    \State \Return $neighbors(x)$
\EndIf
\State \Return $naturalNeighbors(p,x) \bigcup forcedNeighbors(p, x)$ 
\end{algorithmic}
\end{algorithm}

\begin{algorithm}
\caption{Function $naturalNeighbors$}\label{alg:getnn}
\begin{algorithmic}[1]
\Require $p$: parent node, $x$: current node
\If{$direction(p,x) \in H$}
    \State \Return $neighbors(x) \setminus \{p\}$
\EndIf
\State \Return $\{ x + direction(p, x) \}$
\end{algorithmic}
\end{algorithm}

\begin{algorithm}
\caption{Function $forcedNeighbors$}\label{alg:forcedn}
\begin{algorithmic}[1]
\Require $p$: parent node, $x$: current node
\State $forcedNeighbors \gets \emptyset$
\ForAll{$n \in neighbors(x) \setminus (naturalNeighbors(p, x) \bigcup \{p\})$}
    \If{$len( \langle p, x, n \rangle ) < len( \langle p, \ldots , n \rangle \setminus x )$}
        \State add $n$ to $forcedNeighbors$ 
    \EndIf
\EndFor
\State \Return $forcedNeighbors$
\end{algorithmic}
\end{algorithm}

\subsection{Jump Points}

Since the canonical ordering only works if there are no obstacles, JPS8 introduces so called jump points. They are placed on the corners of obstacles so no parts of the state space are missed.

Figure \ref{fig:can-order-obstacle} shows how a canonical ordering can fail to produce a path to the target because of an obstacle.

Upon encountering an obstacle, the canonical path defines a jump point at the obstacle's corner. This point serves as a reset for the ordering, enabling exploration of regions beyond the obstacle, which would otherwise be inaccessible. This is shown in figure \ref{fig:can-order-jp}. These added successors at
the jump points are the forced neighbors. 

\begin{figure}
    \centering
    \begin{subfigure}[t]{0.31\textwidth}
        \vskip 0pt
        \begin{tikzpicture}[scale=1, thick]
          \begin{scope}
            \draw (0, 0) grid (5, 5);
            \node at (0.5,0.5) {A};
            \node at (2.5,3.5) {B};
            \fill[black] (2,1) -- +(0, 1) -- +(1, 1) -- +(1,0) -- cycle;
            \draw[->, arrows={-Triangle[length=0.2cm]}] (0.7,0.5) -- (1.4,0.5);
            \draw[->, arrows={-Triangle[length=0.2cm]}] (1.6,0.5) -- (2.4,0.5);
            \draw[->, arrows={-Triangle[length=0.2cm]}] (2.6,0.5) -- (3.4,0.5);
            \draw[->, arrows={-Triangle[length=0.2cm]}] (3.6,0.5) -- (4.4,0.5);
            
            \draw[->, arrows={-Triangle[length=0.2cm]}] (0.5,0.7) -- (0.5,1.4);
            \draw[->, arrows={-Triangle[length=0.2cm]}] (0.5,1.6) -- (0.5,2.4);
            \draw[->, arrows={-Triangle[length=0.2cm]}] (0.5,2.6) -- (0.5,3.4);
            \draw[->, arrows={-Triangle[length=0.2cm]}] (0.5,3.6) -- (0.5,4.4);
            \draw[->, arrows={-Triangle[length=0.2cm]}] (4.5,0.6) -- (4.5,1.4);
            \draw[->, arrows={-Triangle[length=0.2cm]}] (4.5,1.6) -- (4.5,2.4);
            \draw[->, arrows={-Triangle[length=0.2cm]}] (4.5,2.6) -- (4.5,3.4);
            \draw[->, arrows={-Triangle[length=0.2cm]}] (4.5,3.6) -- (4.5,4.4);
            
            \draw[->, arrows={-Triangle[length=0.2cm]}] (1.5,0.6) -- (1.5,1.4);
            \draw[->, arrows={-Triangle[length=0.2cm]}] (1.5,1.6) -- (1.5,2.4);
            \draw[->, arrows={-Triangle[length=0.2cm]}] (1.5,2.6) -- (1.5,3.4);
            \draw[->, arrows={-Triangle[length=0.2cm]}] (1.5,3.6) -- (1.5,4.4);

            \draw[->, arrows={-Triangle[length=0.2cm]}] (3.5,0.6) -- (3.5,1.4);
            \draw[->, arrows={-Triangle[length=0.2cm]}] (3.5,1.6) -- (3.5,2.4);
            \draw[->, arrows={-Triangle[length=0.2cm]}] (3.5,2.6) -- (3.5,3.4);
            \draw[->, arrows={-Triangle[length=0.2cm]}] (3.5,3.6) -- (3.5,4.4);

          \end{scope}
        \end{tikzpicture}
        \caption{Our canonical ordering misses the path from A to B because of an obstacle.}
        \label{fig:can-order-obstacle}
    \end{subfigure}%
    \hspace{2cm}
    \begin{subfigure}[t]{0.31\textwidth}
        \vskip 0pt
        \begin{tikzpicture}[scale=1, thick]
          \begin{scope}
            \fill[black] (2,1) -- +(0, 1) -- +(1, 1) -- +(1,0) -- cycle;
            \fill[lightgray] (3,2) -- +(0, 1) -- +(1, 1) -- +(1,0) -- cycle;
            \draw (0, 0) grid (5, 5);
            \node at (0.5,0.5) {A};
            \node at (2.5,3.5) {B};
            \draw[->, arrows={-Triangle[length=0.2cm]}] (0.7,0.5) -- (1.4,0.5);
            \draw[->, arrows={-Triangle[length=0.2cm]}] (1.6,0.5) -- (2.4,0.5);
            \draw[->, arrows={-Triangle[length=0.2cm]}] (2.6,0.5) -- (3.4,0.5);
            \draw[->, arrows={-Triangle[length=0.2cm]}] (3.6,0.5) -- (4.4,0.5);
            
            \draw[->, arrows={-Triangle[length=0.2cm]}] (0.5,0.7) -- (0.5,1.4);
            \draw[->, arrows={-Triangle[length=0.2cm]}] (0.5,1.6) -- (0.5,2.4);
            \draw[->, arrows={-Triangle[length=0.2cm]}] (0.5,2.6) -- (0.5,3.4);
            \draw[->, arrows={-Triangle[length=0.2cm]}] (0.5,3.6) -- (0.5,4.4);
            \draw[->, arrows={-Triangle[length=0.2cm]}] (4.5,0.6) -- (4.5,1.4);
            \draw[->, arrows={-Triangle[length=0.2cm]}] (4.5,1.6) -- (4.5,2.4);
            \draw[->, arrows={-Triangle[length=0.2cm]}] (4.5,2.6) -- (4.5,3.4);
            \draw[->, arrows={-Triangle[length=0.2cm]}] (4.5,3.6) -- (4.5,4.4);
            
            \draw[->, arrows={-Triangle[length=0.2cm]}] (1.5,0.6) -- (1.5,1.4);
            \draw[->, arrows={-Triangle[length=0.2cm]}] (1.5,1.6) -- (1.5,2.4);
            \draw[->, arrows={-Triangle[length=0.2cm]}] (1.5,2.6) -- (1.5,3.4);
            \draw[->, arrows={-Triangle[length=0.2cm]}] (1.5,3.6) -- (1.5,4.4);

            \draw[->, arrows={-Triangle[length=0.2cm]}] (3.5,0.6) -- (3.5,1.4);
            \draw[->, arrows={-Triangle[length=0.2cm]}] (3.5,1.6) -- (3.5,2.4);
            \draw[->, arrows={-Triangle[length=0.2cm]}] (3.5,2.6) -- (3.5,3.4);
            \draw[->, arrows={-Triangle[length=0.2cm]}] (3.5,3.6) -- (3.5,4.4);
            
            \draw[->, arrows={-Triangle[length=0.2cm]}] (3.6,2.5) -- (4.4,2.5);
            \draw[->, arrows={-Triangle[length=0.2cm]}] (3.4,2.5) -- (2.6,2.5);
            \draw[->, arrows={-Triangle[length=0.2cm]}] (2.4,2.5) -- (1.6,2.5);
            \draw[->, arrows={-Triangle[length=0.2cm]}] (1.4,2.5) -- (0.6,2.5);
            
            \draw[->, arrows={-Triangle[length=0.2cm]}] (2.5,2.6) -- (2.5,3.3);
          \end{scope}
        \end{tikzpicture}
        \caption{The introduction of a jump point that resets the canonical ordering fixes the search space such that target B is now reachable again.}
        \label{fig:can-order-jp}
    \end{subfigure}%
    \caption{Jump points are introduced to make positions reachable that have become unreachable in our canonical ordering due to obstacles. Obstacles are marked as black squares, while jump points are represented through gray squares.}
\end{figure}
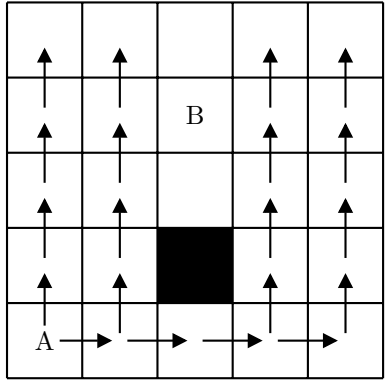
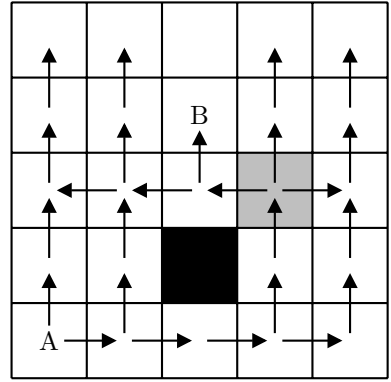

\begin{definition}[Jump Point]\label{def:jp}

Node $y$ is defined as the jump point from node $x$ in direction $d \in D$ if it is the closest (minimizing $k$ in $y=x+kd$) node reachable from $x$ through a single movement in direction $d$ that satisfies at least one of the following conditions:
\begin{enumerate}
\item It is the goal node.
\item $d \in H$
\item It has at least one forced neighbor.
\end{enumerate}
\end{definition}

Algorithm~\ref{alg:successors} shows how to get the successors. This successor function distinguishes JPS4 from A*. In line 2 we get all the nodes that remain after pruning. For each of these nodes we jump ahead in the current direction in line 3 and add the resulting node to the successors (line 4-5).

The actual jumping logic, that implements our definition of jump points is shown in Algorithm~\ref{alg:jump}. We stop jumping whenever we hit an obstacle (lines 3-4). If the direction is horizontal (line 5), the current node has any forced neighbors (line 7), or we found the goal node (line 9), we stop jumping and return the current node as a jump point. Otherwise we proceed jumping in the same direction (line 11).

\begin{algorithm}
\caption{Identify Successors}\label{alg:successors}
\begin{algorithmic}[1]
\Require $p$: parent node, $x$: current node, $g$: goal
\State $successors \gets \emptyset$

\ForAll{$n \in prune(p, x)$}
    \State $n \gets jump(n, direction(x,n), g)$
    \If{$n \neq null$}
        \State add $n$ to $successors$
    \EndIf
\EndFor    
\State \Return $successors$
\end{algorithmic}
\end{algorithm}

\begin{algorithm}
\caption{Function $jump$}\label{alg:jump}
\begin{algorithmic}[1]
\Require $p$: parent node, $d$: direction, $g$: goal
\While {true}
    \State $x \gets step(p, d)$
    \If{$x$ is an obstacle or is outside the grid}
        \State \Return $null$
    \EndIf
    \If{$direction(p,x) \in H$}
        \State \Return $x$
    \EndIf
    \If{$len(forcedNeighbors(p,x)) > 0$}
        \State \Return $x$
    \EndIf
    \If{$x = g$}
        \State \Return $x$
    \EndIf
    
    \State $p \gets x$
\EndWhile
\end{algorithmic}
\end{algorithm}

\subsection{Optimality}

In order to prove that our algorithm always returns an optimal solution, we take the approach by \cite{harabor11} and use turning points to show optimality.

\begin{definition}[Turning Point]\label{def:tp}
A turning point is any node $n_i$ along a path $\langle \dotsc, n_{i-1}, n_i, n_{i+1}, \dotsc \rangle$
where $direction(n_{i -1}, n_i) \neq direction(n_i, n_{i+1})$.\cite{harabor11}
\end{definition}

\begin{algorithm}
\caption{Compute Horizontal First Path}\label{alg:hfpath}
\begin{algorithmic}[1]
\Require $\pi$
\State \Select two consecutive edges in the path, $(n_{k−1}, n_k)$ and $(n_k, n_{k+1})$, where the first edge is a vertical move and the second edge is a horizontal move.
\State \Replace If the node $n'_k$
  located one step horizontally from $n_{k−1}$ is not an obstacle, replace the original two edges with the new pair:
  \begin{itemize}
    \item $(n_{k-1}, n'_k)$: a horizontal move from $n_{k-1}$ to the new node.
    \item $(n'_k, n_{k+1})$: a vertical move from the new node to $n_{k+1}$.
  \end{itemize}
  
\State \Rep Continue steps 1 and 2, searching for and replacing adjacent edge pairs that satisfy the conditions, until no further replacements are possible.
\State \Return $\pi$

\end{algorithmic}
\end{algorithm}

\begin{definition}[Horizontal-First Path]\label{def:hf}
A path $\pi = \langle n_1,n_2, \dotsc, n_k \rangle$ is considered horizontal-first if it does \textbf{not} contain any three consecutive nodes that satisfy the following condition: $n_k$ is a vertical-to-horizontal turning point that could be replaced by a horizontal-to-vertical turning point $n_k'$ such that $\langle n_{k-1},n_k, n_{k+1} \rangle$ becomes $\langle n_{k-1},n_k', n_{k+1} \rangle$ and $len(\langle n_{k-1},n_k, n_{k+1} \rangle) = len(\langle n_{k-1},n_k', n_{k+1} \rangle)$.
\end{definition}

Using Algorithm \ref{alg:hfpath} we can transform any path $\pi$ with optimal length to a symmetric horizontal-first path.

\begin{lemma}\label{lemma:1}
Each turning point along an optimal horizontal-first path $\pi'$
is also a jump point.
\end{lemma}

\begin{proof}
Let $n_k$ be an arbitrary turning point node along $\pi'$.
There are only two kinds of turning points: 
\\

\noindent
\textbf{Horizontal-to-Vertical}: 
Because the direction from $n_{k -1}$ to $n_k$ is horizontal, $n_k$ is a jump point by Definition~\ref{def:jp}.
\\

\noindent
\textbf{Vertical-to-Horizontal}:
We show that there must be an obstacle adjacent to $n_k$ by contradiction:
Assume there is no obstacle adjacent to $n_k$. Then we can replace $n_k$ by a horizontal-to-vertical turning point. This contradicts the assumption that $\pi'$ is horizontal-first. Therefore $n_{k+1}$ must be a forced neighbor of $n_k$. That implies that $n_k$ is a jump point according to Definition~\ref{def:jp}.

\end{proof}

\begin{theorem}
Searching with jump point pruning always returns an optimal solution.
\end{theorem}

\begin{proof}
Given an optimal path $\pi$, let $\pi'$  denote its horizontal-first version obtained via Algorithm \ref{alg:hfpath}.

We divide $\pi'$ into adjacent path segments $\pi' = \pi'_0 + \pi'_1 + \ldots + \pi'_n$ with $\pi'_i = \langle n_0, n_1, \ldots, n_{k-1}, n_k \rangle$ being a subpath such that $\forall{i \in \{i \in \mathbb{N}: 0 \leq i < k-1\}: direction(n_{i}, n_{i+1}) = direction(n_{i+1}, n_{i+2})}$.
So all moves in each $\pi'_i$ are either exclusively up, down, left or right. The above pruning rules never prune the node in movement direction which implies that the algorithm reaches $n_k$ optimally from $n_0$, because any intermediate expansions that might occur don't keep the algorithm from reaching $n_k$ on an optimal path ($\pi'_i$).
Now that we showed the optimality of the individual segments, it remains to show that the connections of the segments are also reached optimally.

Except the start and goal node of $\pi'$, the beginning $n_0$ and end $n_k$ of each $\pi'_i$ are turning points according to Definition \ref{def:tp}. From Lemma \ref{lemma:1} follows that each turning point is also a jump point and is therefore expanded by the algorithm.

The start node of $\pi'$ is always expanded by the algorithm first. The goal node of $\pi'$ is a jump point by definition, so it is also expanded.

\end{proof}

\section{Benchmarks}

We use three different benchmarks from Hierarchical Open Graph \cite{hog}. That data set was published to help improve the consistency of experiments for pathfinding on grid maps \cite{snr}.

The data set is stored as an ASCII grid that encodes different types of terrain for each position. It differentiates between \textit{passable terrain}, \textit{out of bounds}, \textit{trees} (unpassable), \textit{swamp} (passable from regular terrain) and \textit{water} (traversable, but not passable from terrain).
However, in our experiments we consider every position that is not \textit{passable terrain} as unpassable/blocked. Since this is done for both A* and JPS4 the comparison of these two algorithms is still valid.

We use the following data sets:
\begin{itemize}
\item \textbf{Dragon Age: Origins}

The maps are extracted from BioWare's role playing game ``Dragon Age: Origins''. It contains 156 maps and 155620 problems.

\item \textbf{Rooms}

This benchmark contains 40 maps of size 512x512 and 84350 problems.
There are 5 room sizes: 8x8, 16x16, 32x32 and 64x64. For each room size there exist 10 different maps that are completely separated into rooms of the corresponding size.

\item \textbf{Baldurs Gate 2}

The maps are extracted from BioWare's role playing game ``Baldur's Gate 2''. The set contains 75 maps that are scaled to a size of 512x512. It contains 122600 problems.

\item \textbf{Empty Map}

We generated 100 problems for each path length from 1 to 400 (40000 in total). Since there are no 100 different pathfinding problems on empty maps with path lengths of small sizes, we had to repeat those to reach 100 problems per path length.
\end{itemize}

The tests were run on a machine with 32GB of RAM, a 2.90GHz Intel Core i7 processor with 8 cores and Windows 11.
Since the experiments were run in a single Node.js process with limited memory, only 4GB of RAM and a single processor core were used.

\section{Results}

Figure~\ref{fig:speedupresults} shows the performance of JPS4 compared to A* in terms of speedup. Note that we can't compare to the original JPS8 algorithm because it is not applicable to 4-connected grid maps. For larger path lengths we observe a speedup of an order of magnitude for all benchmarks but the empty map.

On obstacle free (empty) maps A* outperforms JPS4. That is because on these kind of maps, A* can get to the target with the heuristic relatively quickly. Usually JPS4 gets its performance advantage by reducing the maximum size and the number of operations (node expansions) of the open and closed lists. If we take an empty map of size 500x500 and run a pathfinding from the top left corner to the lower right corner, we get an optimal path of length 1000. A* has a maximum open-list size of 1495 and takes 3488 open-list operations. JPS4 on the other side only has a maximum open-list size of 2 and only needs 1001 open-list operations. However, JPS4 visits every node on the whole map. It does not add these nodes to the open-list but it still needs to visit more nodes than A*. As a result, A* is still faster.

All considered benchmarks, but especially the Baldur's Gate 2 set have on average a much lower obstacle density for shorter paths than for longer ones. That means that the ratio of obstacles to free positions in the rectangular area that is defined by the start and end node of the search is lower for shorter paths. As we have seen, A* performs better on empty spaces. That explains why A* also outperforms JPS4 on the Baldur's Gate 2 benchmark for path lengths below 60.

The rooms benchmark is structurally very heterogeneous. Therefore the speedup grows almost proportionally with the path length. There is a spike in the graph at about a path length of 1000. The reason for this is that there are significantly fewer problems for path lengths over 1000. For path lengths below 960 there are approximately 8 to 9 problems per path length. Then the amount of problems per path length drops significantly, so that for paths with a length over 1060 there are only 0.8 to 0.9 problems per path length. Since the runtime of problems can be quite different, having only one path per path length adds some noise to the chart.

\section{Conclusion and Future Work}
JPS4 can give a significant speedup over A*. That is especially true for environments with a high obstacle density like rooms which is similar to many generated dungeon maps in video games. Even on benchmarks from real video games like Baldur's Gate 2 and Dragon Age: Origins the performance advantage goes up by an order of magnitude for most path lengths. In the special case of obstacle free maps A* outperforms JPS4. However, the data from Baldur's Gate 2 and Dragon Age: Origins suggests that the obstacle density grows with path length on real problem spaces such that this effect weakens for longer paths. For maps with very low obstacle densities A* might be the better choice. For most other cases JPS4 is an alternative that can provide order of magnitudes faster searches.

Further work could evaluate bounded JPS4. That means the maximum jump length is bounded. It will result in more node expansions but it will at the same time reduce the number of visited nodes. That could weaken the performance problems on empty spaces.

If the density of the map is known before running the algorithm, a pre-computation of the density of certain areas could help to pick either A* or JPS4 at runtime and therefore use them whenever they are suited best.

JPS4 can also be combined with a precomputation approach like \cite{Harabor_Botea_2010}. Since it identifies empty spaces and skips them, it is a good approach to tackle the performance issues of JPS4 on those empty spaces.

\begin{figure}
    \begin{subfigure}[t]{0.47\textwidth}
        \vskip 0pt
        \begin{tikzpicture}[gnuplot]
\tikzset{every node/.append style={scale=0.80}}
\path (0.000,0.000) rectangle (7.500,5.000);
\gpcolor{color=gp lt color border}
\gpsetlinetype{gp lt border}
\gpsetdashtype{gp dt solid}
\gpsetlinewidth{0.50}
\draw[gp path] (1.087,0.967)--(0.907,0.967);
\node[gp node right] at (0.760,0.967) {$0$};
\draw[gp path] (1.087,1.724)--(0.907,1.724);
\node[gp node right] at (0.760,1.724) {$5$};
\draw[gp path] (1.087,2.481)--(0.907,2.481);
\node[gp node right] at (0.760,2.481) {$10$};
\draw[gp path] (1.087,3.239)--(0.907,3.239);
\node[gp node right] at (0.760,3.239) {$15$};
\draw[gp path] (1.087,3.996)--(0.907,3.996);
\node[gp node right] at (0.760,3.996) {$20$};
\draw[gp path] (1.087,4.753)--(0.907,4.753);
\node[gp node right] at (0.760,4.753) {$25$};
\draw[gp path] (1.087,0.967)--(1.087,0.787);
\node[gp node center] at (1.087,0.541) {$0$};
\draw[gp path] (1.750,0.967)--(1.750,0.787);
\node[gp node center] at (1.750,0.541) {$200$};
\draw[gp path] (2.414,0.967)--(2.414,0.787);
\node[gp node center] at (2.414,0.541) {$400$};
\draw[gp path] (3.077,0.967)--(3.077,0.787);
\node[gp node center] at (3.077,0.541) {$600$};
\draw[gp path] (3.741,0.967)--(3.741,0.787);
\node[gp node center] at (3.741,0.541) {$800$};
\draw[gp path] (4.404,0.967)--(4.404,0.787);
\node[gp node center] at (4.404,0.541) {$1000$};
\draw[gp path] (5.068,0.967)--(5.068,0.787);
\node[gp node center] at (5.068,0.541) {$1200$};
\draw[gp path] (5.731,0.967)--(5.731,0.787);
\node[gp node center] at (5.731,0.541) {$1400$};
\draw[gp path] (6.395,0.967)--(6.395,0.787);
\node[gp node center] at (6.395,0.541) {$1600$};
\draw[gp path] (7.058,0.967)--(7.058,0.787);
\node[gp node center] at (7.058,0.541) {$1800$};
\draw[gp path] (1.087,4.753)--(1.087,0.967)--(7.058,0.967)--(7.058,4.753)--cycle;
\draw[gp path] (5.114,1.147)--(5.114,1.639)--(6.911,1.639)--(6.911,1.147)--cycle;
\node[gp node left] at (5.114,1.393) {JPS4 };
\gpcolor{rgb color={0.000,0.000,0.000}}
\draw[gp path] (5.996,1.393)--(6.764,1.393);
\draw[gp path] (1.120,1.059)--(1.153,1.080)--(1.187,1.091)--(1.220,1.093)--(1.253,1.102)%
  --(1.286,1.120)--(1.319,1.154)--(1.352,1.177)--(1.386,1.221)--(1.419,1.261)--(1.452,1.317)%
  --(1.485,1.356)--(1.518,1.439)--(1.551,1.474)--(1.585,1.523)--(1.618,1.579)--(1.651,1.653)%
  --(1.684,1.723)--(1.717,1.781)--(1.750,1.877)--(1.784,1.881)--(1.817,1.952)--(1.850,2.051)%
  --(1.883,2.134)--(1.916,2.172)--(1.949,2.279)--(1.983,2.372)--(2.016,2.283)--(2.049,2.468)%
  --(2.082,2.464)--(2.115,2.568)--(2.149,2.642)--(2.182,2.658)--(2.215,2.693)--(2.248,2.794)%
  --(2.281,2.787)--(2.314,2.814)--(2.348,3.000)--(2.381,3.017)--(2.414,3.196)--(2.447,3.241)%
  --(2.480,3.269)--(2.513,3.394)--(2.547,3.385)--(2.580,3.379)--(2.613,3.362)--(2.646,3.580)%
  --(2.679,3.649)--(2.712,3.597)--(2.746,3.738)--(2.779,3.709)--(2.812,4.004)--(2.845,3.993)%
  --(2.878,4.134)--(2.911,4.113)--(2.945,4.087)--(2.978,4.083)--(3.011,4.265)--(3.044,4.132)%
  --(3.077,4.289)--(3.111,4.145)--(3.144,4.134)--(3.177,4.150)--(3.210,4.286)--(3.243,4.511)%
  --(3.276,4.118)--(3.310,4.084)--(3.343,3.972)--(3.376,4.050)--(3.409,4.104)--(3.442,4.095)%
  --(3.475,4.073)--(3.509,3.766)--(3.542,3.609)--(3.575,3.640)--(3.608,3.539)--(3.641,3.735)%
  --(3.674,3.602)--(3.708,3.660)--(3.741,3.702)--(3.774,3.757)--(3.807,3.748)--(3.840,3.641)%
  --(3.873,3.619)--(3.907,3.764)--(3.940,3.722)--(3.973,3.743)--(4.006,3.524)--(4.039,3.509)%
  --(4.073,3.455)--(4.106,3.690)--(4.139,3.494)--(4.172,3.567)--(4.205,3.595)--(4.238,3.561)%
  --(4.272,3.640)--(4.305,3.617)--(4.338,3.444)--(4.371,3.697)--(4.404,3.664)--(4.437,3.660)%
  --(4.471,3.451)--(4.504,3.518)--(4.537,3.326)--(4.570,3.305)--(4.603,3.322)--(4.636,3.293)%
  --(4.670,3.067)--(4.703,3.015)--(4.736,3.073)--(4.769,3.096)--(4.802,3.025)--(4.835,3.102)%
  --(4.869,3.121)--(4.902,3.030)--(4.935,3.029)--(4.968,3.129)--(5.001,3.070)--(5.034,2.951)%
  --(5.068,2.807)--(5.101,3.005)--(5.134,2.968)--(5.167,3.076)--(5.200,3.088)--(5.234,3.041)%
  --(5.267,3.191)--(5.300,3.267)--(5.333,3.076)--(5.366,3.519)--(5.399,3.241)--(5.433,3.259)%
  --(5.466,3.297)--(5.499,3.105)--(5.532,3.100)--(5.565,3.190)--(5.598,3.060)--(5.632,3.056)%
  --(5.665,2.992)--(5.698,3.072)--(5.731,3.208)--(5.764,2.979)--(5.797,2.962)--(5.831,2.869)%
  --(5.864,2.752)--(5.897,2.634)--(5.930,2.744)--(5.963,2.581)--(5.996,2.485)--(6.030,2.349)%
  --(6.063,2.511)--(6.096,2.511)--(6.129,2.428)--(6.162,2.408)--(6.196,2.457)--(6.229,2.440)%
  --(6.262,2.277)--(6.295,2.378)--(6.328,2.388)--(6.361,2.258)--(6.395,2.036)--(6.428,2.012)%
  --(6.461,2.039)--(6.494,2.022)--(6.527,2.034)--(6.560,2.053)--(6.594,2.043)--(6.627,2.030)%
  --(6.660,2.025)--(6.693,2.053)--(6.726,2.056)--(6.759,2.074)--(6.793,2.053)--(6.826,2.053)%
  --(6.859,2.057);
\gpcolor{color=gp lt color border}
\draw[gp path] (1.087,4.753)--(1.087,0.967)--(7.058,0.967)--(7.058,4.753)--cycle;
\node[gp node center,rotate=-270.0] at (0.085,2.860) {Avg speedup vs A*};
\node[gp node center] at (4.072,0.172) {Path length};
\gpdefrectangularnode{gp plot 1}{\pgfpoint{1.087cm}{0.967cm}}{\pgfpoint{7.058cm}{4.753cm}}
\end{tikzpicture}
        \caption{Baldur's Gate 2 benchmark}
        \label{fig:speedupbg}
    \end{subfigure}%
    \hspace*{0.5cm}   
    \begin{subfigure}[t]{0.47\textwidth}
        \vskip 0pt
        \begin{tikzpicture}[gnuplot]
\tikzset{every node/.append style={scale=0.80}}
\path (0.000,0.000) rectangle (7.500,5.000);
\gpcolor{color=gp lt color border}
\gpsetlinetype{gp lt border}
\gpsetdashtype{gp dt solid}
\gpsetlinewidth{0.50}
\draw[gp path] (1.087,0.967)--(0.907,0.967);
\node[gp node right] at (0.760,0.967) {$0$};
\draw[gp path] (1.087,1.508)--(0.907,1.508);
\node[gp node right] at (0.760,1.508) {$2$};
\draw[gp path] (1.087,2.049)--(0.907,2.049);
\node[gp node right] at (0.760,2.049) {$4$};
\draw[gp path] (1.087,2.590)--(0.907,2.590);
\node[gp node right] at (0.760,2.590) {$6$};
\draw[gp path] (1.087,3.130)--(0.907,3.130);
\node[gp node right] at (0.760,3.130) {$8$};
\draw[gp path] (1.087,3.671)--(0.907,3.671);
\node[gp node right] at (0.760,3.671) {$10$};
\draw[gp path] (1.087,4.212)--(0.907,4.212);
\node[gp node right] at (0.760,4.212) {$12$};
\draw[gp path] (1.087,4.753)--(0.907,4.753);
\node[gp node right] at (0.760,4.753) {$14$};
\draw[gp path] (1.087,0.967)--(1.087,0.787);
\node[gp node center] at (1.087,0.541) {$0$};
\draw[gp path] (1.940,0.967)--(1.940,0.787);
\node[gp node center] at (1.940,0.541) {$500$};
\draw[gp path] (2.793,0.967)--(2.793,0.787);
\node[gp node center] at (2.793,0.541) {$1000$};
\draw[gp path] (3.646,0.967)--(3.646,0.787);
\node[gp node center] at (3.646,0.541) {$1500$};
\draw[gp path] (4.499,0.967)--(4.499,0.787);
\node[gp node center] at (4.499,0.541) {$2000$};
\draw[gp path] (5.352,0.967)--(5.352,0.787);
\node[gp node center] at (5.352,0.541) {$2500$};
\draw[gp path] (6.205,0.967)--(6.205,0.787);
\node[gp node center] at (6.205,0.541) {$3000$};
\draw[gp path] (7.058,0.967)--(7.058,0.787);
\node[gp node center] at (7.058,0.541) {$3500$};
\draw[gp path] (1.087,4.753)--(1.087,0.967)--(7.058,0.967)--(7.058,4.753)--cycle;
\draw[gp path] (5.114,1.147)--(5.114,1.639)--(6.911,1.639)--(6.911,1.147)--cycle;
\node[gp node left] at (5.114,1.393) {JPS4 };
\gpcolor{rgb color={0.000,0.000,0.000}}
\draw[gp path] (5.996,1.393)--(6.764,1.393);
\draw[gp path] (1.104,1.226)--(1.121,1.311)--(1.138,1.354)--(1.155,1.447)--(1.172,1.509)%
  --(1.189,1.607)--(1.206,1.722)--(1.223,1.826)--(1.241,1.974)--(1.258,1.991)--(1.275,2.138)%
  --(1.292,2.304)--(1.309,2.344)--(1.326,2.417)--(1.343,2.534)--(1.360,2.609)--(1.377,2.667)%
  --(1.394,2.758)--(1.411,2.791)--(1.428,2.857)--(1.445,2.951)--(1.462,3.043)--(1.479,3.099)%
  --(1.496,3.137)--(1.514,3.178)--(1.531,3.204)--(1.548,3.271)--(1.565,3.254)--(1.582,3.315)%
  --(1.599,3.286)--(1.616,3.332)--(1.633,3.370)--(1.650,3.419)--(1.667,3.446)--(1.684,3.399)%
  --(1.701,3.473)--(1.718,3.466)--(1.735,3.445)--(1.752,3.478)--(1.769,3.543)--(1.786,3.562)%
  --(1.804,3.586)--(1.821,3.616)--(1.838,3.505)--(1.855,3.676)--(1.872,3.544)--(1.889,3.533)%
  --(1.906,3.573)--(1.923,3.588)--(1.940,3.508)--(1.957,3.522)--(1.974,3.542)--(1.991,3.492)%
  --(2.008,3.559)--(2.025,3.488)--(2.042,3.528)--(2.059,3.464)--(2.076,3.453)--(2.094,3.562)%
  --(2.111,3.487)--(2.128,3.529)--(2.145,3.481)--(2.162,3.564)--(2.179,3.494)--(2.196,3.494)%
  --(2.213,3.560)--(2.230,3.744)--(2.247,3.533)--(2.264,3.585)--(2.281,3.530)--(2.298,3.623)%
  --(2.315,3.671)--(2.332,3.581)--(2.349,3.683)--(2.367,3.744)--(2.384,3.742)--(2.401,3.774)%
  --(2.418,3.758)--(2.435,3.748)--(2.452,3.735)--(2.469,3.891)--(2.486,3.817)--(2.503,3.858)%
  --(2.520,3.746)--(2.537,3.750)--(2.554,3.916)--(2.571,3.669)--(2.588,3.850)--(2.605,3.748)%
  --(2.622,4.022)--(2.639,3.979)--(2.657,4.063)--(2.674,3.885)--(2.691,4.080)--(2.708,4.071)%
  --(2.725,3.968)--(2.742,3.759)--(2.759,3.857)--(2.776,4.043)--(2.793,3.994)--(2.810,3.915)%
  --(2.827,4.247)--(2.844,4.247)--(2.861,4.070)--(2.878,4.288)--(2.895,4.218)--(2.912,3.772)%
  --(2.929,3.778)--(2.947,3.719)--(2.964,3.827)--(2.981,3.680)--(2.998,3.682)--(3.015,3.734)%
  --(3.032,3.870)--(3.049,3.681)--(3.066,3.810)--(3.083,3.986)--(3.100,3.855)--(3.117,3.973)%
  --(3.134,3.788)--(3.151,4.023)--(3.168,4.180)--(3.185,3.794)--(3.202,3.773)--(3.220,4.048)%
  --(3.237,3.821)--(3.254,3.712)--(3.271,3.849)--(3.288,3.866)--(3.305,3.758)--(3.322,3.765)%
  --(3.339,3.819)--(3.356,3.937)--(3.373,3.783)--(3.390,3.968)--(3.407,3.924)--(3.424,3.857)%
  --(3.441,3.710)--(3.458,3.796)--(3.475,3.801)--(3.492,4.021)--(3.510,3.761)--(3.527,3.979)%
  --(3.544,3.744)--(3.561,3.836)--(3.578,3.793)--(3.595,4.215)--(3.612,3.968)--(3.629,3.886)%
  --(3.646,3.843)--(3.663,4.015)--(3.680,3.960)--(3.697,4.230)--(3.714,3.846)--(3.731,3.879)%
  --(3.748,3.958)--(3.765,3.982)--(3.782,3.853)--(3.800,3.856)--(3.817,3.991)--(3.834,3.937)%
  --(3.851,3.953)--(3.868,3.999)--(3.885,4.009)--(3.902,3.851)--(3.919,4.070)--(3.936,3.614)%
  --(3.953,3.564)--(3.970,3.510)--(3.987,3.568)--(4.004,3.551)--(4.021,3.528)--(4.038,3.523)%
  --(4.055,3.519)--(4.073,3.517)--(4.090,3.539)--(4.107,3.546)--(4.124,3.531)--(4.141,3.407)%
  --(4.158,3.534)--(4.175,3.588)--(4.192,3.536)--(4.209,3.362)--(4.226,3.372)--(4.243,3.303)%
  --(4.260,3.544)--(4.277,3.470)--(4.294,3.335)--(4.311,3.511)--(4.328,3.355)--(4.345,3.599)%
  --(4.363,3.568)--(4.380,3.723)--(4.397,3.576)--(4.414,3.782)--(4.431,3.644)--(4.448,3.627)%
  --(4.465,3.761)--(4.482,3.471)--(4.499,3.495)--(4.516,3.589)--(4.533,3.849)--(4.550,3.784)%
  --(4.567,3.724)--(4.584,3.684)--(4.601,3.678)--(4.618,3.539)--(4.635,3.609)--(4.653,3.612)%
  --(4.670,3.650)--(4.687,3.640)--(4.704,3.638)--(4.721,3.659)--(4.738,3.515)--(4.755,3.626)%
  --(4.772,3.718)--(4.789,3.650)--(4.806,3.604)--(4.823,3.602)--(4.840,3.624)--(4.857,3.597)%
  --(4.874,3.488)--(4.891,3.594)--(4.908,3.490)--(4.926,3.469)--(4.943,3.437)--(4.960,3.356)%
  --(4.977,3.389)--(4.994,3.497)--(5.011,3.492)--(5.028,3.575)--(5.045,3.515)--(5.062,3.342)%
  --(5.079,3.588)--(5.096,3.598)--(5.113,3.443)--(5.130,3.579)--(5.147,3.544)--(5.164,3.646)%
  --(5.181,3.654)--(5.198,3.441)--(5.216,3.499)--(5.233,3.762)--(5.250,3.438)--(5.267,3.695)%
  --(5.284,3.307)--(5.301,3.441)--(5.318,3.571)--(5.335,3.679)--(5.352,3.407)--(5.369,3.511)%
  --(5.386,3.445)--(5.403,3.483)--(5.420,3.620)--(5.437,3.448)--(5.454,3.480)--(5.471,3.476)%
  --(5.488,3.575)--(5.506,3.482)--(5.523,3.530)--(5.540,3.452)--(5.557,3.640)--(5.574,3.600)%
  --(5.591,3.496)--(5.608,3.438)--(5.625,3.614)--(5.642,3.609)--(5.659,3.569)--(5.676,3.569)%
  --(5.693,3.476)--(5.710,3.513)--(5.727,3.524)--(5.744,3.559)--(5.761,3.580)--(5.779,3.519)%
  --(5.796,3.419)--(5.813,3.543)--(5.830,3.641)--(5.847,3.693)--(5.864,3.506)--(5.881,3.584)%
  --(5.898,3.506)--(5.915,3.543)--(5.932,3.590)--(5.949,3.499)--(5.966,3.462)--(5.983,3.625)%
  --(6.000,3.480)--(6.017,3.548)--(6.034,3.539)--(6.051,3.526)--(6.069,3.485)--(6.086,3.634)%
  --(6.103,3.552)--(6.120,3.563)--(6.137,3.489)--(6.154,3.592)--(6.171,3.504)--(6.188,3.562)%
  --(6.205,3.607)--(6.222,3.526)--(6.239,3.539)--(6.256,3.510)--(6.273,3.571)--(6.290,3.584)%
  --(6.307,3.554)--(6.324,3.611)--(6.341,3.571)--(6.359,3.584)--(6.376,3.563)--(6.393,3.607)%
  --(6.410,3.653)--(6.427,3.601)--(6.444,3.612)--(6.461,3.658)--(6.478,3.501)--(6.495,3.657)%
  --(6.512,3.490)--(6.529,3.810)--(6.546,3.655)--(6.563,3.607)--(6.580,3.608)--(6.597,3.651)%
  --(6.614,3.603);
\gpcolor{color=gp lt color border}
\draw[gp path] (1.087,4.753)--(1.087,0.967)--(7.058,0.967)--(7.058,4.753)--cycle;
\node[gp node center,rotate=-270.0] at (0.085,2.860) {Avg speedup vs A*};
\node[gp node center] at (4.072,0.172) {Path length};
\gpdefrectangularnode{gp plot 1}{\pgfpoint{1.087cm}{0.967cm}}{\pgfpoint{7.058cm}{4.753cm}}
\end{tikzpicture}
        \caption{Dragon Age: Origins benchmark}
        \label{fig:speedupdao}
    \end{subfigure}%
    \vskip 20pt
    \begin{subfigure}[t]{0.47\textwidth}
        \vskip 0pt
        \begin{tikzpicture}[gnuplot]
\tikzset{every node/.append style={scale=0.80}}
\path (0.000,0.000) rectangle (7.500,5.000);
\gpcolor{color=gp lt color border}
\gpsetlinetype{gp lt border}
\gpsetdashtype{gp dt solid}
\gpsetlinewidth{0.50}
\draw[gp path] (1.087,0.967)--(0.907,0.967);
\node[gp node right] at (0.760,0.967) {$0$};
\draw[gp path] (1.087,1.598)--(0.907,1.598);
\node[gp node right] at (0.760,1.598) {$10$};
\draw[gp path] (1.087,2.229)--(0.907,2.229);
\node[gp node right] at (0.760,2.229) {$20$};
\draw[gp path] (1.087,2.860)--(0.907,2.860);
\node[gp node right] at (0.760,2.860) {$30$};
\draw[gp path] (1.087,3.491)--(0.907,3.491);
\node[gp node right] at (0.760,3.491) {$40$};
\draw[gp path] (1.087,4.122)--(0.907,4.122);
\node[gp node right] at (0.760,4.122) {$50$};
\draw[gp path] (1.087,4.753)--(0.907,4.753);
\node[gp node right] at (0.760,4.753) {$60$};
\draw[gp path] (1.087,0.967)--(1.087,0.787);
\node[gp node center] at (1.087,0.541) {$0$};
\draw[gp path] (2.082,0.967)--(2.082,0.787);
\node[gp node center] at (2.082,0.541) {$200$};
\draw[gp path] (3.077,0.967)--(3.077,0.787);
\node[gp node center] at (3.077,0.541) {$400$};
\draw[gp path] (4.073,0.967)--(4.073,0.787);
\node[gp node center] at (4.073,0.541) {$600$};
\draw[gp path] (5.068,0.967)--(5.068,0.787);
\node[gp node center] at (5.068,0.541) {$800$};
\draw[gp path] (6.063,0.967)--(6.063,0.787);
\node[gp node center] at (6.063,0.541) {$1000$};
\draw[gp path] (7.058,0.967)--(7.058,0.787);
\node[gp node center] at (7.058,0.541) {$1200$};
\draw[gp path] (1.087,4.753)--(1.087,0.967)--(7.058,0.967)--(7.058,4.753)--cycle;
\draw[gp path] (5.114,1.147)--(5.114,1.639)--(6.911,1.639)--(6.911,1.147)--cycle;
\node[gp node left] at (5.114,1.393) {JPS4 };
\gpcolor{rgb color={0.000,0.000,0.000}}
\draw[gp path] (5.996,1.393)--(6.764,1.393);
\draw[gp path] (1.137,1.031)--(1.187,1.076)--(1.236,1.123)--(1.286,1.160)--(1.336,1.213)%
  --(1.386,1.265)--(1.435,1.301)--(1.485,1.360)--(1.535,1.401)--(1.585,1.426)--(1.634,1.490)%
  --(1.684,1.522)--(1.734,1.550)--(1.784,1.575)--(1.833,1.600)--(1.883,1.620)--(1.933,1.636)%
  --(1.983,1.674)--(2.032,1.688)--(2.082,1.734)--(2.132,1.774)--(2.182,1.740)--(2.231,1.795)%
  --(2.281,1.808)--(2.331,1.808)--(2.381,1.844)--(2.430,1.858)--(2.480,1.840)--(2.530,1.957)%
  --(2.580,1.895)--(2.630,1.891)--(2.679,1.906)--(2.729,1.952)--(2.779,1.927)--(2.829,1.989)%
  --(2.878,2.009)--(2.928,2.021)--(2.978,2.033)--(3.028,2.081)--(3.077,2.035)--(3.127,2.040)%
  --(3.177,2.105)--(3.227,2.111)--(3.276,2.091)--(3.326,2.136)--(3.376,2.149)--(3.426,2.136)%
  --(3.475,2.167)--(3.525,2.257)--(3.575,2.234)--(3.625,2.236)--(3.674,2.201)--(3.724,2.270)%
  --(3.774,2.249)--(3.824,2.263)--(3.873,2.240)--(3.923,2.286)--(3.973,2.280)--(4.023,2.379)%
  --(4.073,2.348)--(4.122,2.381)--(4.172,2.372)--(4.222,2.372)--(4.272,2.468)--(4.321,2.425)%
  --(4.371,2.432)--(4.421,2.395)--(4.471,2.529)--(4.520,2.439)--(4.570,2.597)--(4.620,2.602)%
  --(4.670,2.621)--(4.719,2.629)--(4.769,2.588)--(4.819,2.655)--(4.869,2.701)--(4.918,2.712)%
  --(4.968,2.767)--(5.018,2.810)--(5.068,2.839)--(5.117,2.822)--(5.167,2.900)--(5.217,2.902)%
  --(5.267,2.881)--(5.316,2.942)--(5.366,2.993)--(5.416,2.946)--(5.466,2.946)--(5.515,3.073)%
  --(5.565,3.015)--(5.615,3.067)--(5.665,3.085)--(5.715,3.028)--(5.764,3.094)--(5.814,3.062)%
  --(5.864,3.179)--(5.914,3.401)--(5.963,3.617)--(6.013,3.791)--(6.063,3.899)--(6.113,3.963)%
  --(6.162,4.149)--(6.212,4.429)--(6.262,3.874)--(6.312,3.759)--(6.361,3.658)--(6.411,3.707)%
  --(6.461,3.696)--(6.511,3.753)--(6.560,3.603)--(6.610,3.420)--(6.660,2.839)--(6.710,2.805)%
  --(6.759,2.846)--(6.809,3.241);
\gpcolor{color=gp lt color border}
\draw[gp path] (1.087,4.753)--(1.087,0.967)--(7.058,0.967)--(7.058,4.753)--cycle;
\node[gp node center,rotate=-270.0] at (0.085,2.860) {Avg speedup vs A*};
\node[gp node center] at (4.072,0.172) {Path length};
\gpdefrectangularnode{gp plot 1}{\pgfpoint{1.087cm}{0.967cm}}{\pgfpoint{7.058cm}{4.753cm}}
\end{tikzpicture}
        \caption{Rooms benchmark}
        \label{fig:speeduprooms}
    \end{subfigure}%
    \hspace*{0.5cm}   
    \begin{subfigure}[t]{0.47\textwidth}
        \vskip 0pt
        \begin{tikzpicture}[gnuplot]
\tikzset{every node/.append style={scale=0.80}}
\path (0.000,0.000) rectangle (7.500,5.000);
\gpcolor{color=gp lt color border}
\gpsetlinetype{gp lt border}
\gpsetdashtype{gp dt solid}
\gpsetlinewidth{0.50}
\draw[gp path] (1.381,0.967)--(1.201,0.967);
\node[gp node right] at (1.054,0.967) {$0.02$};
\draw[gp path] (1.381,1.508)--(1.201,1.508);
\node[gp node right] at (1.054,1.508) {$0.04$};
\draw[gp path] (1.381,2.049)--(1.201,2.049);
\node[gp node right] at (1.054,2.049) {$0.06$};
\draw[gp path] (1.381,2.590)--(1.201,2.590);
\node[gp node right] at (1.054,2.590) {$0.08$};
\draw[gp path] (1.381,3.130)--(1.201,3.130);
\node[gp node right] at (1.054,3.130) {$0.1$};
\draw[gp path] (1.381,3.671)--(1.201,3.671);
\node[gp node right] at (1.054,3.671) {$0.12$};
\draw[gp path] (1.381,4.212)--(1.201,4.212);
\node[gp node right] at (1.054,4.212) {$0.14$};
\draw[gp path] (1.381,4.753)--(1.201,4.753);
\node[gp node right] at (1.054,4.753) {$0.16$};
\draw[gp path] (1.381,0.967)--(1.381,0.787);
\node[gp node center] at (1.381,0.541) {$0$};
\draw[gp path] (2.091,0.967)--(2.091,0.787);
\node[gp node center] at (2.091,0.541) {$50$};
\draw[gp path] (2.800,0.967)--(2.800,0.787);
\node[gp node center] at (2.800,0.541) {$100$};
\draw[gp path] (3.510,0.967)--(3.510,0.787);
\node[gp node center] at (3.510,0.541) {$150$};
\draw[gp path] (4.220,0.967)--(4.220,0.787);
\node[gp node center] at (4.220,0.541) {$200$};
\draw[gp path] (4.929,0.967)--(4.929,0.787);
\node[gp node center] at (4.929,0.541) {$250$};
\draw[gp path] (5.639,0.967)--(5.639,0.787);
\node[gp node center] at (5.639,0.541) {$300$};
\draw[gp path] (6.348,0.967)--(6.348,0.787);
\node[gp node center] at (6.348,0.541) {$350$};
\draw[gp path] (7.058,0.967)--(7.058,0.787);
\node[gp node center] at (7.058,0.541) {$400$};
\draw[gp path] (1.381,4.753)--(1.381,0.967)--(7.058,0.967)--(7.058,4.753)--cycle;
\draw[gp path] (5.114,4.081)--(5.114,4.573)--(6.911,4.573)--(6.911,4.081)--cycle;
\node[gp node left] at (5.114,4.327) {JPS4 };
\gpcolor{rgb color={0.000,0.000,0.000}}
\draw[gp path] (5.996,4.327)--(6.764,4.327);
\draw[gp path] (1.523,4.276)--(1.665,1.375)--(1.807,1.103)--(1.949,1.051)--(2.091,1.041)%
  --(2.233,1.054)--(2.374,1.060)--(2.516,1.054)--(2.658,1.076)--(2.800,1.100)--(2.942,1.115)%
  --(3.084,1.104)--(3.226,1.100)--(3.368,1.107)--(3.510,1.154)--(3.652,1.177)--(3.794,1.178)%
  --(3.936,1.191)--(4.078,1.209)--(4.220,1.191)--(4.361,1.202)--(4.503,1.207)--(4.645,1.229)%
  --(4.787,1.253)--(4.929,1.261)--(5.071,1.265)--(5.213,1.259)--(5.355,1.238)--(5.497,1.248)%
  --(5.639,1.279)--(5.781,1.290)--(5.923,1.324)--(6.065,1.333)--(6.206,1.343)--(6.348,1.372)%
  --(6.490,1.384)--(6.632,1.379)--(6.774,1.397)--(6.916,1.402)--(7.058,1.410);
\gpcolor{color=gp lt color border}
\draw[gp path] (1.381,4.753)--(1.381,0.967)--(7.058,0.967)--(7.058,4.753)--cycle;
\node[gp node center,rotate=-270.0] at (0.085,2.860) {Avg speedup vs A*};
\node[gp node center] at (4.219,0.172) {Path length};
\gpdefrectangularnode{gp plot 1}{\pgfpoint{1.381cm}{0.967cm}}{\pgfpoint{7.058cm}{4.753cm}}
\end{tikzpicture}
        \caption{Obstacle free map}
        \label{fig:speedupfree}
    \end{subfigure}%

\caption{Average speedup of JPS4 compared to A* grouped by path length.}
\label{fig:speedupresults} 
\end{figure}
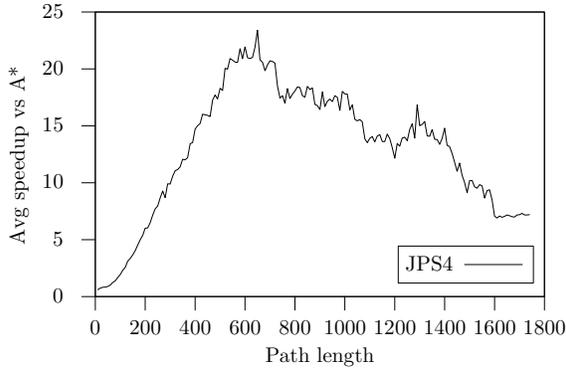
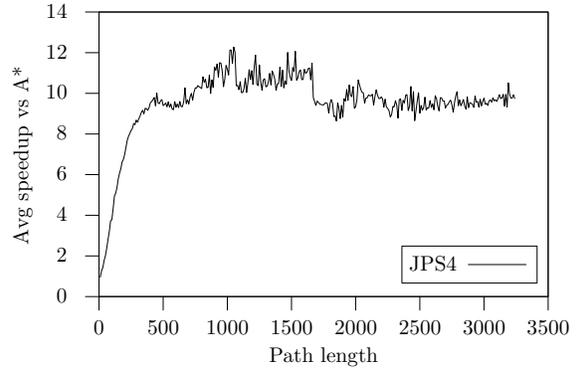
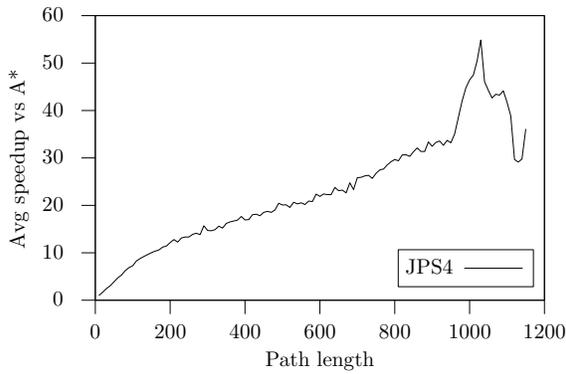
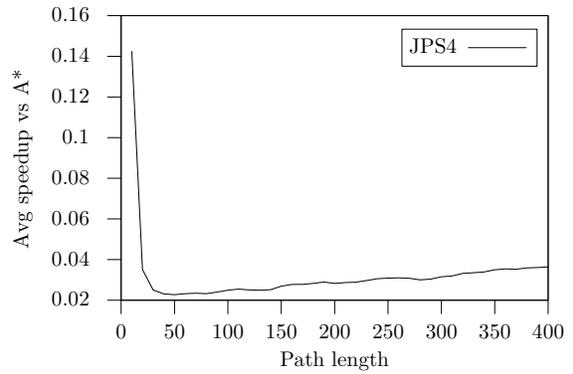

\clearpage

\bibliographystyle{alpha}
\bibliography{sample}

\end{document}